\documentclass[letterpaper, 10pt]{article} % arxiv用

\usepackage{amsmath,amsfonts,bm}
\usepackage{mathtools}

\NewDocumentCommand{\Cot}{oooo}{%
    \mathsf{CoT}\IfNoValueF{#1}{
    	[#1%
  			\IfNoValueF{#2}{ ,#2}%
    		\IfNoValueF{#3}{ ,#3}
    		\IfNoValueF{#4}{ ,#4}
    	]}%
}
\newcommand{\id}{\mathsf{id}}

\newcommand{\Val}{\rmV}
\newcommand{\Key}{\rmK}
\newcommand{\Query}{\rmQ}

\newcommand{\key}{\mathbf{k}}
\newcommand{\query}{\mathbf{q}}

\newcommand{\AND}{\mathsf{AND}}
\newcommand{\OR}{\mathsf{OR}}
\newcommand{\NOT}{\mathsf{NOT}}

\newcommand{\FF}{\mathbf{FF}}

\newcommand{\prob}[1]{\textsc{#1}}
\newcommand{\interleave}[2]{{#1}^\frown{#2}}
\newcommand{\bin}{\mathsf{bin}}
\newcommand{\sbin}{\mathsf{sbin}}
\newcommand{\inner}[2]{\left\langle #1,#2 \right\rangle}
\newcommand{\rds}[1]{\left[#1\right]_s}

\newcommand{\LOOP}{\mathsf{LOOP}}%{\mathsf{nLOOP}}

\newcommand{\CoT}{\mathsf{CoT}}%{\mathsf{nCoT}}

\newcommand{\NC}{\mathsf{NC}}
\newcommand{\TC}{\mathsf{TC}}

\newcommand{\poly}{\mathsf{poly}}

\newcommand{\SIZE}{\mathsf{SIZE}}

\newcommand{\MAJORITY}{\mathsf{MAJORITY}}

\def\eqref#1{equation~\ref{#1}}
\def\1{\bm{1}}

\def\rmK{{\mathbf{K}}}

\def\rmQ{{\mathbf{Q}}}

\def\rmV{{\mathbf{V}}}

\def\ve{{\bm{e}}}

\def\vh{{\bm{h}}}

\def\vz{{\bm{z}}}

\def\mW{{\bm{W}}}

\DeclareMathAlphabet{\mathsfit}{\encodingdefault}{\sfdefault}{m}{sl}
\SetMathAlphabet{\mathsfit}{bold}{\encodingdefault}{\sfdefault}{bx}{n}

\def\gC{{\mathcal{C}}}

\def\gL{{\mathcal{L}}}

\def\gV{{\mathcal{V}}}

\newcommand{\softmax}{\sigma_s}%{\mathrm{softmax}}

\usepackage{natbib}
\usepackage[margin=1in]{geometry}

\usepackage[utf8]{inputenc} % allow utf-8 input
\usepackage[T1]{fontenc}    % use 8-bit T1 fonts
\usepackage{hyperref}       % hyperlinks
\usepackage{url}            % simple URL typesetting
\usepackage{booktabs}       % professional-quality tables
\usepackage{amsfonts}       % blackboard math symbols
\usepackage{nicefrac}       % compact symbols for 1/2, etc.
\usepackage{microtype}      % microtypography
\usepackage{xcolor}         % colors
\usepackage{microtype}
\usepackage{graphicx}
\usepackage{subcaption}
\usepackage{hyperref}
\usepackage{here}
\usepackage{cancel}
\usepackage{amsmath}
\usepackage{amssymb}
\usepackage{mathtools}
\usepackage{amsthm}
\usepackage{multirow}
\usepackage{bbding}
\usepackage{tabularx}
\usepackage{hhline}
\usepackage{enumitem}
\usepackage{cleveref}
\usepackage{bbm}

\usepackage{algorithm}
\usepackage{algpseudocode}

\usepackage{amsmath}
\usepackage{amssymb}
\usepackage{mathtools}
\usepackage{amsthm}
\usepackage{thmtools}
\usepackage{thm-restate}

\theoremstyle{plain}
\newtheorem{theorem}{Theorem}[section]
\newtheorem{proposition}[theorem]{Proposition}
\newtheorem{lemma}[theorem]{Lemma}

\theoremstyle{definition}
\newtheorem{definition}[theorem]{Definition}
\newtheorem{assumption}[theorem]{Assumption}
\theoremstyle{remark}

\newtheorem*{theorem*}{Theorem}

\title{To CoT or To Loop? A Formal Comparison Between Chain-of-Thought and Looped Transformers}

\author{%
Kevin Xu\thanks{kevinxu@g.ecc.u-tokyo.ac.jp} \quad Issei Sato\thanks{sato@g.ecc.u-tokyo.ac.jp} \\
The University of Tokyo
}

\if0
\author{%
  Kevin Xu \\
  Department of Computer Science\\
  The University of Tokyo\\
  \texttt{kevinxu@g.ecc.u-tokyo.ac.jp} \\
  \And
  Issei Sato \\
  Department of Computer Science\\
  The University of Tokyo\\
  \texttt{sato@g.ecc.u-tokyo.ac.jp} \\
}
\fi

\begin{document}

\maketitle

\begin{center}
    \color{red}
    \Large This paper is outdated.\\
    For the latest version, please visit 
\href{https://arxiv.org/abs/2509.25239}{arXiv:2509.25239}.
\end{center}

\begin{abstract}
Chain-of-Thought (CoT) and Looped Transformers have been shown to empirically improve performance on reasoning tasks and to theoretically enhance expressivity by recursively increasing the number of computational steps. However, their comparative capabilities are still not well understood.
In this paper, we provide a formal analysis of their respective strengths and limitations. We show that Looped Transformers can efficiently simulate parallel computations for deterministic tasks, which we formalize as evaluation over directed acyclic graphs. In contrast, CoT with stochastic decoding excels at approximate inference for compositional structures, namely self-reducible problems.
These separations suggest the tasks for which depth-driven recursion is more suitable, thereby offering practical cues for choosing between reasoning paradigms.
%and those where probabilistic self-correction may be advantageous,
%Code is available at~\url{https://github.com/kevin671/cot-vs-loop}.
\end{abstract}
% self-correction: LLMs refine their responses based on feedback to their previous outputs

\section{Introduction}
Transformer-based large language models (LLMs)\citep{vaswani2017attention} have achieved impressive performance across a wide range of tasks and have recently been extended to complex reasoning tasks.
Rather than directly predicting final answers, prompting LLMs to generate intermediate reasoning steps, known as \emph{Chain-of-Thought (CoT)} prompting\citep{wei2022chain}, has been shown to significantly enhance reasoning capabilities.
This has given rise to a broader paradigm in which inference-time compute is increasingly leveraged to support complex reasoning~\citep{welleck2024from}.

This naturally raises the question: \emph{why is CoT prompting effective for complex reasoning tasks?}
Recent work has addressed this question by framing reasoning as a computational problem and analyzing its computational complexity.
\cite{feng2023towards} showed that certain classes of problems that cannot be solved with standard Transformers become solvable with CoT.
Subsequent studies analyzed the complexity class of CoT with respect to the number of steps, compared to models such as Turing machines~\citep{merrill2024the}, Boolean circuits~\citep{li2024chain}, and probabilistic Turing machines (PTMs)~\citep{nowak2024}. In summary, these findings underscore that the strength of CoT lies in its ability to increase the effective number of computational steps.

Recently, \emph{Looped Transformer (Looped TF)} was proposed as an architecture that recursively applies fixed-size Transformer layers by feeding their outputs back into their inputs~\citep{dehghani2018universal}.
Similar to CoT, Looped TF increases the number of computational steps.
However, they achieve this increase implicitly via architectural recursion rather than through explicit token-level decoding.
The recursive structure has been shown to enhance expressivity: theoretical results establish Turing completeness\citep{giannou2023looped} and universal approximation~\citep{xu2025on}, while empirical work demonstrates improved performance on reasoning tasks~\citep{saunshi2025reasoning}.

The existence of two distinct approaches raises a fundamental question:
\begin{center}
\emph{What is the separation between Chain-of-Thought and Looped Transformer?}
\end{center}
This question has motivated comparisons of their expressivity \emph{under the same number of steps or loops}. 
\citet{saunshi2025reasoning} showed an expressivity inclusion of CoT within TF, and \citet{merrill2025little} showed a separation favoring Looped TF in the logarithmic regime.
While these results suggest a superiority of looped models over CoT, several important questions remain unanswered:
\begin{itemize}[label={}, leftmargin=*]
    \item \emph{Does separation exist beyond the logarithmic regime?} \textrightarrow\ Affirmative; see \Cref{sec:det}.
    \item \emph{Are Looped Transformers universally more expressive than CoT?} \textrightarrow\ Negative; see \Cref{sec:prob}.
\end{itemize}

In this paper, we aim to clarify the distinct characteristics of Looped TF and CoT by analyzing two problem settings, thereby addressing these questions.
\begin{itemize}[leftmargin=*]% , nosep]
    \item \textbf{Looped TF for Parallel Computation:} For deterministic settings, we analyze computational structures as directed acyclic graphs and reveal a fundamental difference: CoT is inherently sequential, whereas Looped TF enables parallel solutions. Furthermore, we establish a separation under the same polylogarithmic bound on the number of steps or loops.
    \item \textbf{Approximate Inference with CoT:} We study the problem of approximately sampling from the uniform distribution over structured solutions, formalized via \emph{self-reducibility}~\citep{Schnorr1976}. The complexity of this task aligns with that of approximate counting, highlighting the advantage of probabilistic CoT over deterministic Looped TF. Moreover, we show that CoT is not only expressive, but that even a \emph{weak} CoT model, when combined with \emph{self-consistency}, can successfully solve the task.
\end{itemize}

%%%%%%%%%%%%%%%%%%%%%%%%%%%%%%%%%%%%%%%%%%%%%%%%%%%%%%%%%%%%
\section{Preliminaries and Notation}\label{sec:prel}

%\paragraph{Transformer-based Architectures}  
We consider models based on Transformer blocks that are applied recursively. Specifically, we analyze two architectures that differ in how intermediate computations are structured and reused across steps. %Before introducing these architectures, we briefly outline our assumptions about the Transformer components.
We assume \textbf{saturated attention}, where attention weights are uniformly assigned to the tokens with the highest scores. For the feedforward computation, we consider either a standard fully connected layer or a mixture-of-experts (MoE) layer~\citep{shazeer2017, csordas2024moeut}.  
Formal definitions of the Transformer block are deferred to \Cref{app:tf}.

\subsection{Chain-of-Thought}
A \emph{decoder-only Transformer} with causal masking is used for \emph{next-token prediction}. In CoT, the model is allowed to generate intermediate steps before producing the final answer. At each step, it decodes a single token conditioned on the entire current sequence, appends it to the input, and proceeds to the next step. The intermediate computation steps are explicitly represented as tokens.
\begin{definition}[Informal: CoT]
Let \( \gV \) be a vocabulary, and let \( f : \gV^* \to \gV \) denote an \emph{autoregressive} decoder-only Transformer, where decoding is performed either deterministically or stochastically as a next-token prediction.
Given an input sequence \( x = (x_1, \dots, x_n) \in \gV^n \), define the initial state as \(f^0(x) \coloneq x.\)
Then, for each step \( k \), the CoT is defined recursively as
\begin{equation}
f^{k+1}(x) \coloneq f^k(x) \cdot f(f^k(x)),
\end{equation}
where \( \cdot \) denotes sequence concatenation.  
The final output after \( T(n) \) steps is defined as the last \( m \) tokens of \( f^{T(n)}(x) \), corresponding to the answer.
\end{definition}

\subsection{Looped Transformer}
%In contrast to CoT, 
Looped TF applies a non-causal Transformer block repeatedly to a fixed-length sequence of internal representations.
The intermediate computations are maintained in the embedding space, rather than as explicit tokens, at each position and are iteratively refined across loop iterations.
\begin{definition}[Informal: Looped TF]
Let \( f: \mathbb{N} \times \mathbb{Q}^{d \times n} \to \mathbb{Q}^{d \times n} \) denote a standard Transformer layer without causal masking, where the first argument \( k \in \mathbb{N} \) allows injecting loop-dependent positional encodings.
Here, \( \mathbb{Q} \) denotes the data type, and \( d \) is the embedding dimension.  
Given an input sequence \( x = (x_1, \dots, x_n) \in \mathcal{V}^n \), define the initial hidden state as
\(f^0(x) \coloneq (e(x_1), \dots, e(x_n)) \in \mathbb{Q}^{d \times n},\) where \( e : \mathcal{V} \to \mathbb{Q}^d \) is a token-wise embedding function.
Then, for each loop \( k \), the Looped TF is defined recursively as
\begin{equation}
f^{k+1}(x) \coloneq f(k,f^k(x)).
\end{equation}
After \( T(n) \) loop iterations, the final output is obtained by applying a linear projection and deterministic decoding (e.g., \(\arg\max\)) to each position in \( f^{T(n)}(x) \), extracting the last \( m \le n \) tokens.
\end{definition}

\subsection{Notation}
%\paragraph{Notation} 
For any \( n \in \mathbb{N}^+ \), we write \( [n] \coloneq \{1, 2, \ldots, n\} \). For functions \( f, g : \mathbb{N} \to \mathbb{N} \), we write \( f(n) \in O(g(n)) \) iff there exist constants \( c > 0 \) and \( n_0 \in \mathbb{N} \) such that \( f(n) \le c \cdot g(n) \) for all \( n \ge n_0 \). We denote by \( \poly(n) \) the set of functions that grow at most polynomially, i.e., \(\poly(n) \coloneq \left\{ f : \mathbb{N} \to \mathbb{N} \;\middle|\; \exists k \in \mathbb{N},\; \exists c > 0,\; \forall n \in \mathbb{N},\; f(n) \le c \cdot n^k \right\}.\) Given two complexity classes \(\gC_1\) and \(\gC_2\), we write \(\gC_1 \setminus \gC_2\) to denote the set of decision problems (referred to as \emph{languages} in computational complexity theory) that belong to \(\gC_1\) but not to \(\gC_2\).

\section{General Problem Setting}
This section begins by reviewing the problem settings considered in prior work, as summarized in \Cref{table:related}, and identifying the remaining open questions. We then formalize the two problem settings addressed in this study, specifying their underlying assumptions and objectives.

\subsection{Related Work}% and Remaining Problems}
\begin{table}
  \caption{Summary of theoretical analyses on expressivity of CoT and Looped TF.}
  \label{table:related}
  \centering
    \begin{tabular}{l|l|cl|c}
    \toprule
    % 計算クラスやseparationがあるかどうかを軸として追加する（チェックかどうか）
    % CoTとLoopedとDet.とPro → CとLにするとか文字数工夫
    % problem settingが長い
    Paper & Model & Type & Problem Setting & Class \\ %  or Class \\
    \midrule
    \cite{prystawski2023why} & CoT & Pro. & Bayesian network & - \\
    \cite{feng2023towards} & CoT & Det. & Mathematics \& Decision-making & \checkmark \\ % \checkmark
    \cite{merrill2024the} & CoT & Det. & Automata \& Turing machine & \checkmark \\
    %& Automata, Turing machine \\ 
    \cite{li2024chain} & CoT & Det. & Boolean circuit & \checkmark \\ 
    \cite{nowak2024} & CoT & Pro. & Language modeling (PTM) & \checkmark \\
    \cite{saunshi2025reasoning} & Looped & Det. & Non-looped, CoT, and Automata & - \\
    \cite{merrill2025little} & Looped & Det. & Automata \& Graph connectivity & \checkmark \\
    %\cite{kim2025metastable} & CoT & Pro. & Search (metastable dynamics) & - \\
    \midrule
    \multirow{2}{*}{\textbf{Ours}} & Looped & Det. & Directed acyclic graph & \checkmark \\
    & \&CoT & Pro. & FPRAS \& FPAUS (self-reducibility) & \checkmark \\
    % / consistency Self-correction / consistency
    \bottomrule
  \end{tabular}
\end{table}
\paragraph{Expressivity of CoT}
Prior analyses of CoT can be broadly categorized into \emph{deterministic} and \emph{probabilistic} settings, each corresponding to a distinct problem formulation and focusing on different aspects of expressivity.
In the deterministic setting, \citet{feng2023towards} showed that CoT enables solving certain function tasks, such as mathematical and dynamic programming problems.
Subsequent works clarified the computational classes corresponding to different numbers of reasoning steps by drawing comparisons to formal models of computation: \citet{merrill2024the} analyzed CoT through its correspondence with automata and Turing machines, while \citet{li2024chain} established a connection to Boolean circuits, illustrating how CoT’s expressivity scales with the number of reasoning steps.
In contrast, the probabilistic setting focuses on the model’s ability to represent and reason over distributions.
\citet{prystawski2023why} studied CoT with sampling-based decoding for probabilistic inference, and \citet{nowak2024} further investigated its capacity to model distributions over strings in language modeling, establishing a connection to PTMs.

\paragraph{Expressivity of Looped TF}
Looped TF inherently operates deterministically and has therefore primarily been studied on deterministic tasks, with their expressivity analyzed in comparison to deterministic CoT under an equal number of steps or loops.
\citet{saunshi2025reasoning} showed that a Looped TF with \( T \) iterations can simulate \( T \) steps of a non-looped or CoT, thereby demonstrating an inclusion of CoT within the expressivity of Looped TF. Furthermore, \citet{merrill2025little} established that Looped TF with only \( \log n \) iterations can solve problems such as graph connectivity, which is \(\mathsf{NL}\)-complete, whereas nonuniform CoT with \( \log n \) steps remain within \(\mathsf{TC}^0\), thereby proving a strict separation in complexity class at the logarithmic scale.

\subsection{Open Questions and Motivations}
We now outline open questions in the existing literature, separately for deterministic and probabilistic settings, each of which motivates the problem formulations studied in this work.

\begin{itemize}[leftmargin=*]

\item \textbf{Deterministic Setting:}
Although Looped TF has been studied on specific tasks, its overall computational characteristics remain unclear. In contrast to CoT, which has been characterized as sequential nature~\citep{feng2023towards, merrill2024the, li2024chain}, the computational properties of Looped TF remain unexplored, in particular, their relationship to parallel computation~\citep{sanford2024transformers}.
Furthermore, while the expressiveness of CoT has been analyzed across various step regimes, that of Looped TF beyond the logarithmic regime, as well as any formal separation from CoT, has yet to be established.

\item \textbf{Probabilistic Setting:}  
While probabilistic CoT has been investigated primarily for distribution modeling, its applicability to solving deterministic problems remains insufficiently understood, limiting direct comparisons with Looped TF.
Moreover, there remains a gap between the theoretical capabilities of CoT as PTMs~\citep{nowak2024} and the empirical behavior of practical LLMs: such models are typically trained to minimize loss up to a bounded error, and inference-time techniques are necessary to exploit their probabilistic behavior for improved performance.
%While probabilistic CoT havs been studied for distribution modeling, their ability for complex reasoning remains underexplored. 
%In particular, inference-time techniques that leverage the probabilistic behavior of CoT, such as majority voting~\citep{wang2023selfconsistency} or iterative self-correction~\citep{madaan2023selfrefine}, have shown empirical gains in LLMs.  
%However, the theoretical understanding of such enhancements remains limited. Only recently have studies begun to examine the role of inference-time methods~\citep{wu2025inference, kim2025metastable}. In particular, the mechanisms by which probabilistic reasoning improves accuracy are still not well understood.
\end{itemize}

\subsection{Our Problem Settings}
We consider two fundamental classes of computational problems: the \emph{deterministic problem} and the \emph{probabilistic problem}, elaborated in \Cref{sec:det} and \Cref{sec:prob}, respectively. This section provides high-level definitions that specify the input–output structure in each setting; precise formulations are deferred to the corresponding sections. These problem classes differ in several fundamental aspects, such as output uniqueness, the presence of a target distribution, and the inclusion of accuracy parameters.
\begin{definition}[Deterministic Problem]\label{def:det}
Let \( f : \mathcal{X} \to \mathcal{Y} \) be a \textbf{target function}.  
A model \( M \) solves the deterministic problem for \( f \) if, for every input \( x \in \mathcal{X} \), it holds that \( M(x) = f(x) \).
\end{definition}
\begin{definition}[Probabilistic Problem]\label{def:sample}
Let \( R \subseteq \mathcal{X} \times \mathcal{Y} \) be a relation such that each input \( x \in \mathcal{X} \) may have multiple valid outputs defined by \(R(x) \coloneq \{ y \in \mathcal{Y} \mid (x, y) \in R \}.\)  
A \textbf{target distribution} \( p(\cdot \mid x) \) is specified over the set \( R(x) \).  %
The \emph{probabilistic problem} is defined with respect to this distribution and is evaluated against explicitly specified accuracy or failure parameters provided as part of the input. %(e.g., approximate sampling within a specified total variation distance).
\end{definition}
%
%\textbf{Remark.} These problem classes differ in several fundamental aspects, such as output uniqueness, the presence of a target distribution, and the inclusion of accuracy parameters.
%The last aspect reflects the behavior of LLMs at inference time, which allocates compute during inference to meet a desired level of reasoning accuracy.
%These considerations motivate our focus on the expressivity of such algorithms in the probabilistic reasoning problem.

%%%%%%%%%%%%%%%%%%%%%%%%%%%%%%%%%%%%%%%%%%%%%%%%%%%%%%%%%%%%
\section{Looped Transformer Enables Efficient Parallel Solutions}\label{sec:det}
This section analyzes the fundamental computational differences between CoT and Looped TF in the deterministic setting.
We formalize deterministic computations using \emph{computation graphs}, as introduced in \Cref{sec:d:task}.
In \Cref{sec:d:low}, we show how CoT and Looped TF differ in processing the same computation graphs, highlighting their respective strengths. Finally, \Cref{sec:d:sep} presents a separation in their computational complexity in the polylogarithmic regime.

\subsection{Problem Setting}\label{sec:d:task}
\begin{figure}[t]
  \centering
  \includegraphics[width=\linewidth]{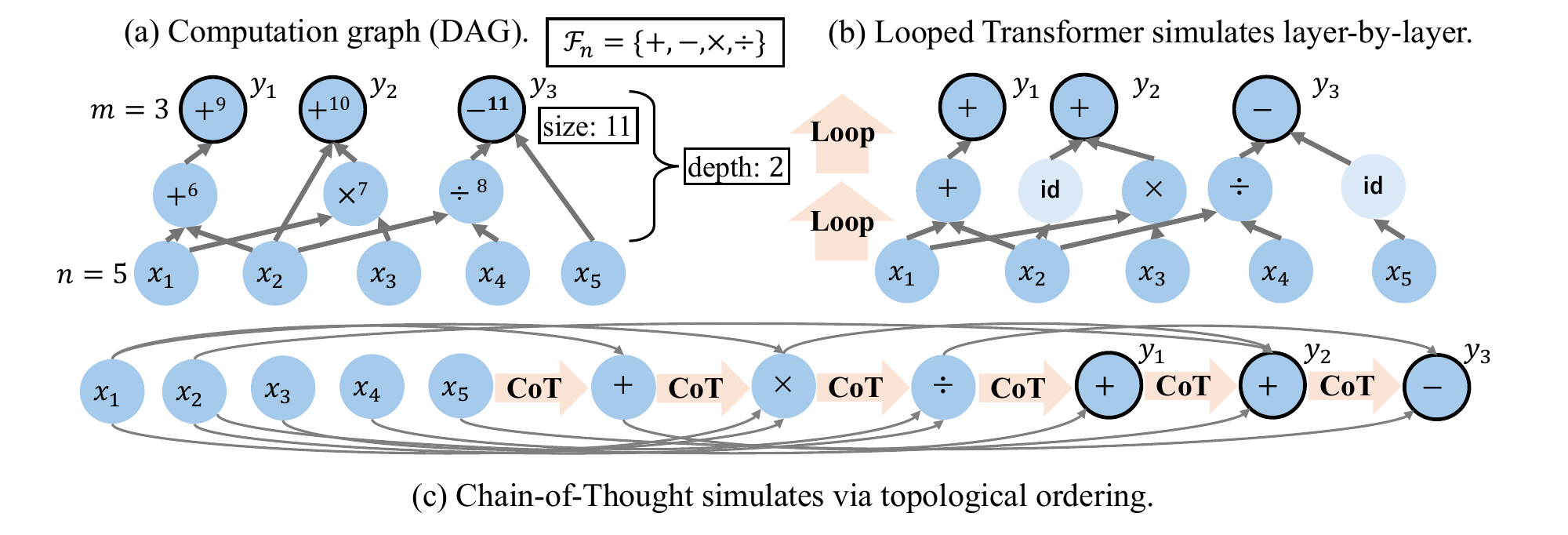}
  \caption{Comparison of simulating strategies for computing a DAG.
    (a) A computation graph. (b) A Looped TF simulates the graph in a layer-wise and parallel fashion, requiring a number of loops proportional to the depth of the graph. (c) A CoT simulates the computation sequentially in topological order, requiring a number of steps proportional to the size of the graph.}
  \label{fig:graph}
\end{figure}
To analyze the behavior of CoT and looped models, we formalize deterministic computations as directed acyclic graphs (DAGs), where each node corresponds to a local function and each edge indicates a data dependency.
This abstraction is natural and general, as straight-line programs can be expressed as DAGs~\citep{aho1972optimization}. An illustrative example is shown in Fig.\,\ref{fig:graph}(a).
\begin{definition}[Computation Graph]
Let \( \Sigma \) be a finite alphabet. A \emph{computation graph} is a directed acyclic graph \( G_n = (V_n, E_n, \mathcal{F}_n) \) that defines a function \( G_n: \Sigma^n \to \Sigma^m \), where \( \mathcal{F}_n = \{ f_n : \Sigma^{\ell_n} \to \Sigma \mid \ell_n \in \mathbb{N} \} \) is a finite set of local functions associated with each node. The graph \(G_n\) consists of (i) input nodes with in-degree $0$,  (ii) function nodes labeled by \( f \in \mathcal{F}_n \) whose in-degree matches the arity of the corresponding function, and (iii) \( m \) output nodes \( O \subseteq V_n \) with out-degree $0$.
The overall function is obtained by evaluating the graph in topological order.
The \emph{size} of the graph is \( |V_n| \), and its \emph{depth} is the length of the longest path from an input node to an output node.
We assume that the size of the computation graph \( |V_n| \) is bounded by a polynomial in the input length \( n \).
\end{definition} % Let \( \mathcal{F}_n \) be a finite set of functions \( \{ f_n \mid f_n : \Sigma^{k(n)} \to \Sigma,\ \text{for some } k(n) \in \mathbb{N} \} \). 
\textbf{Remark.} Prior work has analyzed CoT from DAGs, but in specific instances: \citet{feng2023towards} showed that dynamic programming tasks can be simulated via DAGs, and \citet{li2024chain} related CoT to Boolean circuits, a subclass of DAGs.
We extend this view by modeling general computations as DAGs, yielding a unified framework for analyzing both CoT and Looped TF and clarifying how they simulate such computations in fundamentally different ways.

\subsection{Simulating Computation DAGs: CoT Needs Size-scaled Steps and Looped Transformers Need Depth-scaled Loops}\label{sec:d:low}
To characterize the expressiveness of CoT and Looped TF, we analyze how each simulates computation graphs. We begin by formalizing the assumptions under which computations can be implemented by Transformer-based models (see \Cref{app:def} for details).
\begin{assumption}[Polynomially-Efficient Approximation (cf.~\citealp{feng2023towards})]\label{ass:0}
Each function in the computation graph can be approximated by a feedforward neural network of polynomial size.
\end{assumption}
\begin{assumption}[Bounded Fan-in]\label{ass:1}%[Attention with Bounded Fan-in]
The graph has a bounded fan-in. %, so that a constant number of attention heads can access all required inputs.
\end{assumption}
We now formalize the expressivity of CoT in simulating DAGs under these assumptions:
\begin{theorem}[CoT for DAGs]\label{thm:cot_lower}
Let \( \{G_n\}_{n \in \mathbb{N}} \) be a family of computation graphs, each computing a function \( G_n : \Sigma^n \to \Sigma^{m(n)} \), and satisfying Assumptions~\ref{ass:0} and~\ref{ass:1}.  
Then, for each \( n \in \mathbb{N} \), there exists a log-precision CoT with parameter size bounded by \( \poly(n) \), such that for every input \( x \in \Sigma^n \), the model outputs \( G_n(x) \) in a number of steps proportional to the ``size'' of the graph \( G_n \).
\end{theorem}
We then analyze Looped TF, which requires an additional assumption. Unlike CoT, Looped TF does not extend the input sequence; consequently, the amount of information that can be retained at each loop iteration is limited. %(See \Cref{app:def} for details.)
\begin{assumption}[Linear Parallel Space Complexity]\label{ass:3}
The amount of memory that must be maintained concurrently to evaluate the computation graph in parallel is \( O(n) \).
\end{assumption}
\begin{assumption}[Output Bound]\label{ass:4}
The computation graph contains at most \( n \) output nodes.
\end{assumption}
Under these constraints, we formalize the expressivity of looped models in simulating DAGs.
\begin{theorem}[Looped TF for DAGs]\label{thm:loop_lower}
Let \( \{G_n\}_{n \in \mathbb{N}} \) be a family of computation graphs, each computing a function \( G_n : \Sigma^n \to \Sigma^{m(n)} \), and satisfying Assumptions~\ref{ass:0},~\ref{ass:1},~\ref{ass:3}, and~\ref{ass:4}.  
Then, for each \( n \in \mathbb{N} \), there exists a log-precision Looped TF with parameter size bounded by \( \poly(n) \),  such that for every input \( x \in \Sigma^n \), it computes \( G_n(x) \) using a number of loops proportional to the ``depth'' of the graph \( G_n \).
\end{theorem}
\textbf{Remark.} Proofs are given in \Cref{app:det}, with illustrations in Fig.\,\ref{fig:graph}. Informally, \Cref{thm:cot_lower} shows that CoT simulates computations by decoding nodes in topological order, whereas \Cref{thm:loop_lower} shows that Looped TF simulates computations in a layer-wise and parallel fashion.
This contrast reflects complementary strengths: CoT employs intermediate steps as scratchpad memory, supporting general computation without structural constraints. In contrast, Looped TF exploits structural parallelism to achieve greater efficiency, but requires additional assumptions.

\subsection{Separation between CoT and Looped Transformers in Polylogarithmic Time}\label{sec:d:sep}
To compare the computational complexity of CoT and Looped TF, we analyze them with respect to Boolean circuits, a model of computation represented by DAGs.
In particular, we restrict our attention to languages.
We begin by defining the complexity classes of CoT, as defined in~\cite{li2024chain}, and Looped TF.
\begin{definition}[\(\CoT\) and \(\LOOP\)]
Let \(\CoT[T(n), d(n), p(n)]\) (resp.\ \(\LOOP[T(n), d(n), p(n)]\)) denote the set of languages \( \gL: \{0,1\}^* \to \{0,1\} \) for which there exists a CoT (resp.\ Looped TF) \( M_n \) for each input size \(n\), with embedding size \( O(d(n)) \) and \( O(p(n)) \) bits of precision, such that for all \( x \in \{0,1\}^n \), the final output token, after \( O(T(n)) \) decode steps (resp.\ loops), is \(\gL(x)\), i.e., \(
M_n(x) = \gL(x), \text{ for all } x \in \{0,1\}^n.
\)
\end{definition}
To characterize the computational power of these models, we refer to standard nonuniform circuit complexity classes. In particular, we focus on the classes \( \NC^k \) and \( \TC^k \), sing Boolean and threshold circuits, respectively, which capture efficient parallel computation~\citep{arora2009computational}.
%using Boolean and threshold circuits, respectively~\citep{vollmer1999introduction}. %using Boolean and threshold circuits, respectively~\citep{vollmer1999introduction, arora2009computational}.
%Of particular interest are classes such as (nonuniform) \( \mathsf{NC}^k \) and \( \mathsf{TC}^k \), which consist of circuits with polynomial size and polylogarithmic depth, representing \emph{efficient parallel computation}~\citep{arora2009computational}. hoge, boolen, \cite{vollmer1999introduction}, parallel computation
\begin{definition}[\(\TC^k\) and \(\NC^k\)]
For each integer \( k \ge 0 \), the nonuniform circuit classes are defined as:
\begin{itemize}[leftmargin=*, nosep]
    \item \(\NC^k\) is the class of languages decidable by uniform Boolean circuits of polynomial size and depth \( O(\log^k n) \), with bounded fan-in AND, OR, and NOT gates.
    \item \(\TC^k\) is the class of languages decidable by uniform threshold circuits of polynomial size and depth \( O(\log^k n) \), with unbounded fan-in majority (threshold) gates.
\end{itemize}
\end{definition}
Since Boolean circuits allow polynomial-size, we assume a polynomial embedding size.
%To distinguish the computational power of CoT and Looped Transformers, we analyze their corresponding circuits.
Prior work has shown that CoT is equivalent to Boolean circuits composed of \(\TC^0\) gates, with circuit ``size'' proportional to the number of CoT steps. This class \(\SIZE^{\TC^0}\) is defined in \Cref{def:tcsize}.
\begin{theorem}[\citealp{li2024chain}]
$\forall T(n)\in\poly(n), \CoT[T(n),\poly(n),\log{n}] = \SIZE^{\TC^0}[T(n)+1]$
\end{theorem}
In contrast, we show that Looped TF corresponds to Boolean circuits whose ``depth'' is proportional to the number of loop iterations.
\begin{theorem}[Looped and Circuits]\label{thm:loop}
Let \( T(n) \in \mathrm{poly}(n) \).  
Then, a decision problem \( \gL \) belongs to the class \( \LOOP[T(n), \mathrm{poly}(n), \log n] \) if and only if \( \gL \) is decidable by nonuniform families of threshold circuits of polynomial size and depth \( T(n) \).  
Similarly, \( \gL \in \LOOP[T(n), \mathrm{poly}(n), 1] \) if \( \gL \) is decidable by nonuniform families of Boolean circuits of polynomial size and depth \( T(n) \).
\end{theorem}
These results formally separate their expressive power under standard complexity assumptions.
\begin{theorem}
\(\TC^{k-1}\subsetneq\NC^
k\Rightarrow\CoT[\log^k{n},\poly(n),\log{n}] \subsetneq \LOOP[\log^k{n},\poly(n),1].\)
\end{theorem}
\begin{theorem}\label{thm:nonunisep}
\(\TC^{k-1}\subsetneq\TC^
k\Rightarrow\CoT[\log^k{n},\poly(n),\log{n}] \subsetneq \LOOP[\log^k{n},\poly(n),\log{n}].\)
\end{theorem}
\begin{proof}
It suffices to show that \( \CoT[\log^k n, \poly(n), \log n] \subseteq \TC^{k-1} \) for any \(k\).
For any input, a CoT with \(\log n\) steps yields only \(2^{\log n} = \poly(n)\) possible execution traces, each of which can be simulated in \(\TC^0\).  
Thus, we divide the total \(\log^k n\) steps into \( \log^{k-1} n \) blocks, each consisting of \(\log n\) steps.  
We then simulate each block in parallel using \(\TC^0\), and iterate this process over \(\log^{k-1} n\) layers.  
This layered simulation can be performed within \(\TC^{k-1}\), completing the proof.
\end{proof}
\textbf{Remark.} These results show that, under the same polylogarithmic number of steps or loops, Looped TF can simulate parallel computations more efficiently than CoT.  
This highlights the inherent parallelism of looped models, in contrast to the sequential nature of CoT.
\paragraph{Instance Examples} There are several concrete problems in (nonuniform) \(\NC^2\): solving linear equations, testing for perfect matching in bipartite graphs, and membership testing for fixed context-free grammars, denoted by \( \prob{LIN} \), \( \prob{MATCH} \), and \( \prob{FCFL} \), respectively (see \Cref{app:nc2} for details). Then we have:
\begin{proposition}\label{thm:log_loop}
$\prob{LIN}, \prob{MATCH}, \prob{FCFL} \in\LOOP[\log^2{n},\poly(n), 1] \setminus \CoT[\log^2{n},\poly(n), \log{n}].$
\end{proposition}

%\subsection{CoT Enables Unbounded Memoization for Sequential decision-making}

%%%%%%%%%%%%%%%%%%%%%%%%%%%%%%%%%%%%%%%%%%%%%%%%%%%%%%%%%%%%
\section{Chain-of-Thought Enables Probabilistic Reasoning}\label{sec:prob} 
In this section, we investigate the expressiveness of CoT with stochastic decoding for probabilistic tasks.
\Cref{sec:p:pro} introduces the task formulation and defines the model.
\Cref{sec:p:sep} analyzes the expressive power of CoT in an idealized setting and establishes a formal separation from Looped TF.
Finally, \Cref{sec:p:self} shows that even \emph{weak} CoT, when combined with \emph{self-consistency}, exhibits sufficient expressive power.

\subsection{Problem Setting}\label{sec:p:pro}

\paragraph{Structured Reasoning as Self-Reducibility} 
Complex reasoning tasks, such as solving puzzles, proving theorems, or planning actions, often require producing a globally consistent output through a sequence of interdependent local decisions, each addressing a subproblem.
To formalize this compositional structure, we adopt the notion of \emph{self-reducibility}~\citep{Schnorr1976}.
\begin{definition}[Informal: Self-reducibility~\citep{Schnorr1976}]
A relation \( R \subseteq \Sigma^* \times \Sigma^* \) is \emph{self-reducible} if any solution \( y = y_1 \dots y_n \) to input \( x \) can be constructed incrementally, such that a prefix \( y_{1:\sigma(x)} \) determines a subproblem \( \psi(x, y_{1:\sigma(x)}) \) whose solution completes \( y \). See \Cref{def:self-red} for a formal definition.
\end{definition}
For example, consider the Boolean satisfiability problem (SAT). A satisfying assignment can be constructed incrementally, where each partial assignment defines a simpler subproblem over the remaining variables. This process illustrates the canonical form of self-reducibility.

\paragraph{Task Formulation}  
We formalize the computational task of approximately sampling a structured solution.
\begin{definition}[\( \varepsilon \)-Approximate Sampling (FPAUS)]  
Let \( R \subseteq \Sigma^* \times \Sigma^* \) be a self-reducible relation, and let \( p(\cdot \mid x) \) denote the uniform distribution over the solution set \( R(x) \coloneq \{ y \mid (x, y) \in R \} \).  
The task of \emph{\( \varepsilon \)-approximate sampling}, also known as \emph{almost uniform generation}, is to, for any given error parameter \( \varepsilon > 0 \), output a distribution \( q(\cdot \mid x) \) such that for all \( x \in \Sigma^* \),
\begin{equation}
\| q(\cdot \mid x) - p(\cdot \mid x) \|_{\mathrm{TV}} \leq \varepsilon,
\end{equation}
where \( \| \cdot \|_{\mathrm{TV}} \) denotes the total variation distance.
\end{definition}

\paragraph{Probabilistic CoT Model} 
We consider a CoT equipped with sampling-based decoding, where stochastic branching at intermediate steps allows for diverse reasoning paths and autoregressively defines conditional distributions over output tokens. This formulation aligns with the behavior of LLMs, which explore a space of possible reasoning trajectories and produce varied outputs across different sampling runs.
\begin{definition}[Probabilistic CoT]\label{def:prob-cot}
Let \( V \) be a finite vocabulary and \( V^* \) its Kleene closure.  
A \emph{probabilistic CoT} stochastically generates, given an input \( x \in V^* \), a sequence of output blocks of the form
\[
\langle r_1 \rangle \langle e \rangle \langle y_1 \rangle \langle e' \rangle \;
\langle r_2 \rangle \langle e \rangle \langle y_2 \rangle \langle e' \rangle \;
\cdots \;
\langle r_m \rangle \langle e \rangle \langle y_m \rangle \langle e' \rangle,
\]
where each \( r_i \in V^* \) is a reasoning path, \( e, e' \in V \) are special delimiter tokens, yielding the final output string \( y_1 \dots y_m \).
%with \( |r_i| \leq L(|x|) \) for some \( L(n) \in \poly(n) \)
%
The generation proceeds \emph{autoregressively} via \emph{next-token prediction}: for each \( i \in [m] \), the model generates a reasoning step \( r_i \) followed by an output token \( y_i \), conditioned on the input \( x \), previous outputs \( y_{<i} \), and prior reasoning steps \( r_{<i} \). We denote by \( \pi \) the next-token conditional distribution defined by the model.
\end{definition}

\paragraph{Approximate Counting}
Approximate sampling is closely related to \emph{approximate counting}~\citep{jerrum1986random}.  
For a relation \( R \), define \( f(x) \coloneq \lvert R(x) \rvert \), the number of solutions \( y \) such that \( (x, y) \in R \). A \emph{Fully Polynomial-Time Approximation Scheme (FPTAS)} is a deterministic algorithm that, for any input \( x \) and \( \varepsilon > 0 \), outputs \( \hat{f}(x) \) such that \((1 - \varepsilon) f(x) \le \hat{f}(x) \le (1 + \varepsilon) f(x),\)
in time polynomial in \( |x| \) and \( 1/\varepsilon \).
A \emph{Fully Polynomial Randomized Approximation Scheme (FPRAS)} is a randomized version that additionally takes \( \delta > 0 \), and satisfies the same guarantee with probability at least \( 1 - \delta \). See \Cref{app:approx} for details.

\subsection{CoT is More Expressive than Looped TF in Approximate Inference}
\label{sec:p:sep}
Whereas the previous section showed the superiority of Looped TF over deterministic CoT, we now present a separation in the opposite direction.
Since deterministic models cannot perform sampling, we consider approximate counting; for self-reducible relations, it is equivalent to approximate sampling~\citep{jerrum1986random}, and thus serves as a proxy for evaluating probabilistic inference capabilities.
\begin{theorem}[\citealp{jerrum1986random}]
For self-reducible relations, approximate counting and approximate sampling are polynomial-time equivalent.
\end{theorem}

In computational theory, approximate sampling or counting algorithms are formalized as PTMs that run in polynomial time~\citep{jerrum1986random}. Prior work has shown that probabilistic CoT can simulate PTM.
\begin{lemma}[Informal: CoT as PTM~\citep{nowak2024}]\label{thm:psep}
Any polynomial-time probabilistic Turing machine can be simulated by a probabilistic CoT, such that the model reproduces the same output distribution and halts within a polynomial number of decoding steps.
\end{lemma}

%\paragraph{Separation between Probabilistic CoT and Looped TF}
%We now compare CoT with Looped TF.
Let \( \mathsf{FPRAS} \) and \( \mathsf{FPTAS} \) denote the function classes admitting an FPRAS and an FPTAS, respectively.  
Some problems admit an FPRAS, but no FPTAS is currently known for them~\citep{jerrum2004polynomial}.
Intuitively, allowing failure (but can be small for arbitrarily high confidence) makes such problems tractable via randomized methods. %, such as Markov chain Monte Carlo (MCMC).
Assuming this gap, we identify the first task for which CoT is more powerful than Looped Transformer:
\begin{theorem}\label{thm:p:sep}
Assuming that \( \mathsf{FPTAS} \ne \mathsf{FPRAS} \), there exists a relation \( R \) such that a probabilistic CoT can approximate the count \( |R(x)| \) within any desired relative error \( \varepsilon > 0 \) with high probability in polynomial time, whereas no (deterministic) Looped TF can achieve the same approximation within a polynomial number of loop iterations.
\end{theorem}

\subsection{Expressiveness of Weak CoT}\label{sec:p:self} 
While \Cref{thm:psep} analyzes an idealized setting where a CoT, viewed as a PTM, is assumed to receive explicit error parameters, such assumptions diverge from the behavior of practical LLMs.  
In practice, models are trained within bounded error, but do not offer explicit control over the accuracy of their outputs.  %
In this section, we examine the expressiveness of \emph{weak} CoT models, incorporating inference-time strategies such as majority voting. The central question is: \emph{Can weak CoT be approximate samplers?}

\subsubsection{Formal Setting for Weak CoT}
% 先に幾つかのassumptionを与える
We consider models trained to minimize the next-token prediction loss as follows.
\begin{definition}[Token-Wise Cross-Entropy Loss under Teacher Forcing]
For each input \( x \in \Sigma^* \), let \( p(y_i \mid x, y_{<i}) \) denote the target conditional distribution over the next token \( y_i \in \Sigma \) at position \( i \),  
and let \( \pi(y_i \mid x, y_{<i}, r_i) \) denote the model's prediction conditioned on a reasoning trajectory \( r_i \).
Then, the token-wise cross-entropy loss under teacher forcing is defined as
\begin{equation}
\mathcal{L}_{\mathrm{CE}}(x, y_{<i}, r_i) 
\coloneqq 
\mathbb{E}_{y_i \sim p(\cdot \mid x, y_{<i})} \left[
- \log \pi(y_i \mid x, y_{<i}, r_i)
\right],
\end{equation}
where \( r_i \) denotes a reasoning trajectory sampled from \( \pi(r_i \mid x, y_{<i}) \).
\end{definition}
We assume that each reasoning path \( r_i \) depends only on the preceding output tokens \( y_{<i} \), and is conditionally independent of the previous reasoning steps \( r_{<i} \).
\begin{assumption}[Conditional Independence of Reasoning Paths]
For each position \( i \in [m] \), the reasoning path \( r_i \) is conditionally independent of prior reasoning steps given the input and preceding outputs: \(\pi(r_i \mid x, r_{<i}, y_{<i}) = \pi(r_i \mid x, y_{<i}).\)
\end{assumption}
We further assume access to repeated samples of reasoning paths:
\begin{assumption}[Sampling Independence]\label{assumption:sampling}
We assume access to an unbounded number of independent samples of reasoning paths, i.e., for each \( i \in [m] \), one can sample \( r_i^{(1)}, r_i^{(2)}, \dots \sim \pi(\cdot \mid x, y_{<i}) \) independently.
\end{assumption}
Then, we define the convergence of pertaining. 
\begin{definition}[\((\alpha, \gamma)\)-Weak Token-Wise Learning]\label{assumption:token-ce}
Let \( \alpha: \mathbb{N}\to\mathbb{R}_{>0} \) be any positive function, and let \( \gamma \in (0, \tfrac{1}{2}) \).
We assume that for each token position \( i \in [m] \), and for any input \( x \in \Sigma^* \) and prefix \( y_{<i} \), the model satisfies the following  
\emph{token-wise conditional cross-entropy loss bound under teacher forcing}:  
\begin{equation}
\Pr_{r_{i} \sim \pi(\cdot \mid x, y_{<i})} \left[
\mathcal{L}_{\mathrm{CE}}(x, y_{<i}, r_i) \le \alpha(|x|) %\frac{1}{|R(x)|}
\right] \ge \frac{1}{2} + \gamma.
\end{equation}
\end{definition}
\textbf{Remark.}
This definition formalizes a weak yet meaningful correctness criterion for token-wise predictions in reasoning models.
The task is assumed to be sufficiently difficult that directly predicting \( y_i \) from \( x \) and \( y_{<i} \) alone does not yield low cross-entropy.  
Introducing an intermediate reasoning step \( r_i \) is thus necessary to compensate for this gap.
It guarantees the model achieves sufficiently low cross-entropy loss with probability at least \(\tfrac{1}{2} + \gamma\).
This raises the following questions:
\begin{itemize}[leftmargin=*]
\item \emph{Can a weak CoT support inference with arbitrarily high confidence?}
\item \emph{Can a weak CoT be transformed into an approximate sampler for any desired accuracy?}
\end{itemize}
These questions are addressed in the subsequent sections, with formal proofs deferred to \Cref{app:prob} and key ideas shown in Fig.\,\ref{fig:weak}.
\begin{figure}[t]
  \centering
  \includegraphics[width=\linewidth]{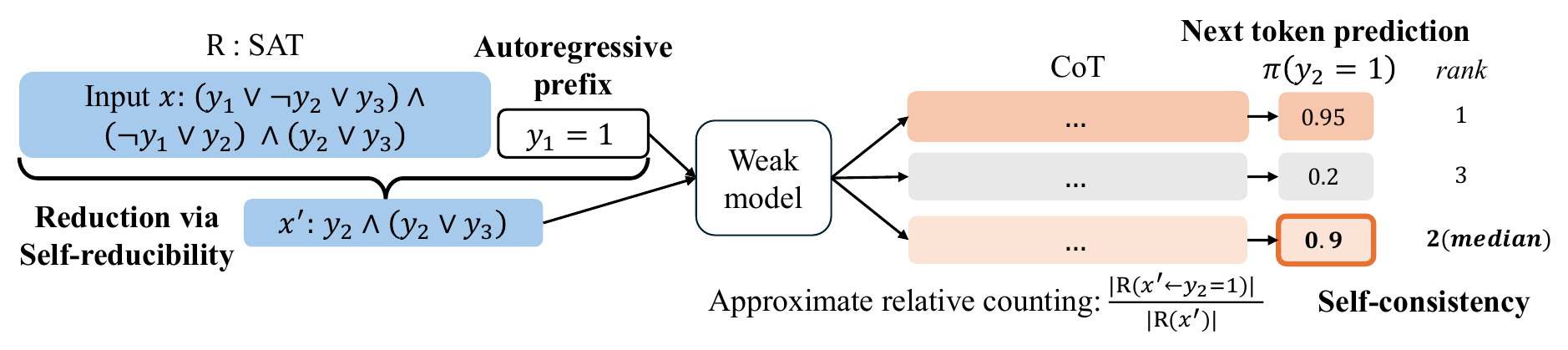}
  \caption{
    Illustration of next-token prediction in CoT as approximate relative counting over a self-reducible relation. Given a SAT instance \( x \), self-reducibility enables reduction to a subproblem \( x' \) based on the autoregressive prefix (e.g., \( y_1 = 1 \)). The probability \( \pi(y_2 = 1) \) is approximated via relative counting, with self-consistency aggregating predictions from multiple reasoning paths.}
  \label{fig:weak}
\end{figure}

\subsubsection{Confidence Amplification via Self-Consistency}\label{sec:cons} 
We show that \emph{self-consistency}\citep{wang2023selfconsistency} can enhance the confidence. While prior work\citep{wu2025inference} focused on classification tasks, we consider settings in which the output is a probability distribution.
\begin{lemma}[Coordinate-wise Self-Consistency]\label{lemma:self}
Suppose that a CoT model \( \pi \) achieves \((\alpha, \gamma)\)-weak token-wise learning with respect to the target distribution \( p(y_i \mid x, y_{<i}) \), for some function \( \alpha : \mathbb{N} \to \mathbb{R}_{>0} \) and constant \( \gamma \in (0, \tfrac{1}{2}) \).  
Let \( \delta \in (0,1) \).  
Fix any token position \( i \in [m] \). Let \( r_i^{(1)}, \dots, r_i^{(T)} \sim \pi( \cdot \mid x, y_{<i}) \) be \(T\) independent reasoning paths,  
and let \( \pi^{(t)} = \pi(Y_i \mid x, y_{<i}, r_i^{(t)}) \in \Delta^{|\mathcal{V}|} \) denote the predicted distribution at position \( i \) for each path.
Define the consensus prediction \( \pi^\star \in \mathbb{R}^{|\mathcal{V}|} \) by taking the coordinate-wise median:
\begin{equation}
\pi^\star(y) = \operatorname{median} \left( \pi^{(1)}(y), \dots, \pi^{(T)}(y) \right) \quad \text{for each } y \in \mathcal{V}.
\end{equation}
Then, for \(T \ge \frac{2}{\gamma^2} \log\left( \frac{|\mathcal{V}|}{\delta} \right)\), with probability at least \(1 - \delta\), the \textit{self-consistent} prediction satisfies
\begin{equation}
\left| p(y_i \mid x, y_{<i}) - \pi^\star(y_i) \right| \leq O\left( \sqrt{ \alpha(|x|) } \right) \quad \text{for all } y_i \in \mathcal{V}.
\end{equation}
That is, self-consistency improves the reliability of predictions by aggregating multiple reasoning paths.
\end{lemma}
% \mathbb{E}_{y_i \sim p(y_i \mid x, y_{<i})} \left[- \log \pi^\star(y_i)\right]

\subsubsection{Approximate Sampling via Self-Reducibility}
We then show that the error of weak CoT can be corrected, yielding an approximate sampler. We begin by noting that next-token predictions become well-defined tasks when the underlying relation is self-reducible, thus justifying the realizability assumption underlying weak token-wise learning, as shown in Fig.\,\ref{fig:weak}.
\begin{proposition}[Reducing Next-Token Prediction to Relative Counting via Self-Reducibility]\label{prop:selfreduce}
Let \( R \subseteq \Sigma^* \times \Sigma^* \) be a self-reducible relation as defined in \Cref{def:self-red}, and suppose further that \( \sigma(x) = 1 \) for all \( x \in \Sigma^* \).
Let \( p(y \mid x) \) denote the uniform distribution over the solution set \( R(x) \coloneqq \{ y \mid (x, y) \in R \} \).  
Then, for CoT, the next-token prediction at each position \( i \) corresponds to estimating the relative size of the subproblem
\begin{equation}
p(y_i \mid y_{<i} = \bar{y}_{<i})
\propto \left| R(x') \right|,
%\left| R(x, \bar{y}_{<i} \cdot y_i) \right| = \left| R(x') \right|.
\end{equation}
where \( x' = \psi(x, \bar{y}_{<i} \cdot y_i) \), i.e., reduces to an instance of the same relation \( R \).
%Moreover, by self-reducibility, there exists a polynomial-time computable transformation \( x' = \psi(x, \bar{y}_{<i} \cdot y_i) \) such that
%\(\left| R(x, \bar{y}_{<i} \cdot y_i) \right| = \left| R(x') \right|,\)
\end{proposition}

We now present the result on the ability of a weak CoT to serve as an FPAUS.
\begin{theorem}[CoT as Approximate Sampler]\label{thm:approx_sampler}
Let \( R \) be a self-reducible relation, and let \( p(\cdot \mid x) \) denote the uniform distribution over the solution set \( R(x) \coloneqq \{ y \in \Sigma^* \mid (x, y) \in R \} \).
There exists a function \( \alpha : \mathbb{N} \to \mathbb{R}_{>0} \) such that if a CoT model \( \pi \), with vocabulary \( \mathcal{V} = \Sigma \), satisfies \((\alpha, \gamma)\)-weak token-wise learning with respect to \( p \) for some constant \( \gamma \in (0, \tfrac{1}{2}) \),  
then there exists a randomized algorithm \( \mathcal{A} \) with oracle access to \( \pi \) such that, for a given error parameter \( \varepsilon > 0 \), \( \mathcal{A} \) outputs a distribution \( q^{\pi}_\varepsilon(\cdot \mid x) \) satisfying
\begin{equation}
\bigl\|\,q^{\pi}_\varepsilon(\cdot\mid x)\;-\;
      p(\cdot\mid x)\bigr\|_{\mathrm{TV}}
\;\;\le\;\;\varepsilon.
\end{equation}
Moreover, the total running time of \( \mathcal{A} \) is bounded by a polynomial in \( (|x|, \log(1/\varepsilon)) \).
\end{theorem}
% やっぱりpの全体の分布と次単語予測をちゃんと分けないと...
% 仮定がそもそもどうやって学習可能性になっているかを説明する
\textbf{Remark.} The above result applies to any relation for which the CoT satisfies the weak token-wise learning assumption. However, the result becomes computationally meaningful and practically implementable when the relation is \emph{self-reducible}, as shown in \Cref{prop:selfreduce}. 
%since in this case, each next-token prediction corresponds to solving a well-defined subproblem within a recursive decomposition,  

\section{Conclusion}\label{sec:con} 
We formally compared CoT and Looped TF, highlighting their respective strengths. Specifically, Looped TF enables efficient parallel computation for deterministic problems, while CoT enables probabilistic inference.
Specifically, we introduced a novel perspective grounded in the complexity of approximate computation.
These findings offer guidance on whether to CoT or to Loop when scaling inference compute.
Future work includes examining the limitations of Looped TF on deterministic tasks, particularly those arising from their fixed space complexity. 
Another important direction is to compare their FLOPs.

\bibliography{ref}
\bibliographystyle{neurips_2024}

%%%%%%%%%%%%%%%%%%%%%%%%%%%%%%%%%%%%%%%%%%%%%%%%%%%%%%%%%%%%

\appendix

\end{document}